\documentclass[sigconf,nonacm]{acmart}

\renewcommand\footnotetextcopyrightpermission[1]{}
\usepackage{microtype}
\usepackage{graphicx}
\usepackage{subfigure}   
\usepackage{booktabs}
\usepackage{multirow}
\usepackage{mathtools}
\usepackage[export]{adjustbox}
\usepackage{float}
\usepackage[capitalize,noabbrev]{cleveref}

\theoremstyle{plain}
\newtheorem{theorem}{Theorem}[section]

\newtheorem{lemma}[theorem]{Lemma}

\theoremstyle{definition}

\theoremstyle{remark}

\begin{document}

\title[GDNSQ: Low-bit Quantization]{GDNSQ: Gradual Differentiable Noise Scale Quantization for Low-bit Neural Networks}

\author{Sergey Salishev}
\authornote{Equal contribution.}
\affiliation{
  \institution{aifoundry.org}
  \city{San Francisco}
  \state{CA}
  \country{USA}
}
\email{salishev[at]aifoundry.org}

\author{Ian Akhremchik}
\authornotemark[1]
\affiliation{
  \institution{Ainekko Co.}
  \city{San Francisco}
  \state{CA}
  \country{USA}
}
\email{jan[at]nekko.ai}



\begin{abstract}
Quantized neural networks can be viewed as a chain of noisy channels, where rounding in each layer reduces capacity as bit-width shrinks; the floating-point (FP) checkpoint sets the maximum input rate. We track capacity dynamics as the average bit-width decreases and identify resulting quantization bottlenecks by casting fine-tuning as a smooth, constrained optimization problem. Our approach employs a fully differentiable Straight-Through Estimator (STE) with learnable bit-width, noise scale and clamp bounds, and enforces a target bit-width via an exterior-point penalty; mild metric smoothing (via distillation) stabilizes training. Despite its simplicity, the method attains competitive accuracy down to the extreme W1A1 setting while retaining the efficiency of STE.
\end{abstract}

\maketitle

\section{Introduction}
Quantization of neural networks (NNs) is a key problem across many applications. 
For small models, it is feasible to implement them as electronic circuits or to map them to low-bit, dedicated hardware (e.g., neural accelerators or TPUs). 
For larger models, quantization compresses weights and accelerates inference, reducing resource usage. 
These practical needs have spurred a surge of work on the topic.

Mathematically, quantization leads to a mixed-integer programming problem. 
Direct search is infeasible at scale, so two dominant practical approaches are used: 
(i) \emph{post-training quantization} (PTQ), which rounds the weights of a floating-point (FP) model; and 
(ii) \emph{quantization-aware training} (QAT), which fine-tunes (or trains) a quantized model and typically achieves higher accuracy at the cost of training time. 
Most QAT methods fall into two families. 
The first, the \emph{straight-through estimator} (STE), was popularized by Bengio et al.~\cite{bengio2013estimating} (originally attributed to Hinton), and provides a surrogate gradient for discontinuous quantizers. 
The second replaces hard steps by smooth surrogates whose steepness is gradually increased; this idea traces back at least to backpropagation~\cite{rumelhart1986learning}. 
A third line explicitly tackles the mixed-integer problem (e.g., SLB~\cite{yang2020searching}) but is less common due to complexity.

Despite the perception that STE degrades at very low bit-widths (e.g., W1A1, W2A2), we argue that its potential is not fully exploited. 
We introduce a simple but effective variant that smooths the STE while retaining its efficiency.

\paragraph{Coding-theoretic view.}
Our approach starts from the observation that a softmax classifier can be interpreted as a (noisy) maximum-likelihood (ML) decoder over a code in a latent ``code space'': each layer acts as a noisy channel whose rounding acts like channel noise; the trained weights realize a forward-error-correcting (FEC) code. This view is similar to Shannon's noisy channel coding metaphor~\cite{shannon1949}. 
Classical links between classification and error-correcting codes go back to Error-Correcting Output Codes (ECOC)~\cite{DietterichBakiri1995ECOC,AllweinSchapireSinger2000}, where class labels are embedded as codewords and prediction corresponds to decoding. 
If the network is robust to input noise (and hence the FP checkpoint corresponds to well-separated codewords), then under a smooth decision metric the effective codebook should evolve smoothly as rounding noise increases (i.e., as channel capacity decreases), making it amenable to optimization by SGD.

\paragraph{Making the metric smooth.}
We employ knowledge distillation~\cite{hinton2015distilling} from the FP teacher using the Jeffreys (\emph{symmetrized} Kullback--Leibler) divergence to smooth the decision metric during quantization; distillation during QAT using cross-entropy or KL is standard practice~\cite{shin2020knowledge,zhang2021learning}. We specifically adopt Jeffreys divergence because, under Jeffreys-loss softmax decoder and a binary asymmetric channel (BAC) noise model, the resulting decision rule reduces to minimizing Hamming distance (see Appendix).

\paragraph{Controlling noise smoothly.}
We gradually increase effective quantization noise via a differentiable STE-style scaling  that combines LSQ~\cite{esser2019learned} with PACT~\cite{choi2018pact}, yielding a smooth path in effective bit-width. The details explaining the STE derivation from the minibatch SGD assumptions, and possible STE modifications are discussed in the Appendix.

\paragraph{Stochastic-approximation view of STE dithering.}
Among STE variants, we adopt DoReFa-style gradient dithering \cite{zhou2016dorefa}, which aligns with the stochastic-approximation viewpoint: the backward of the rounding nonlinearity is replaced by Bernoulli “probing” noise, directly analogous to SPSA \cite{spall1992spsa,spall2003stochsearch}. Under standard regularity conditions, such probes yield asymptotically unbiased descent directions with robustness to noise; variance can be reduced via iterate averaging \cite{polyak1992averaging,granichin2003randomized}, and, in our case, mini-batch SGD. Thus, our Bernoulli STE acts as an SPSA-style randomized gradient surrogate.

\paragraph{Constrained optimization.}
These ingredients lead to a smooth constrained optimization problem where the target bit-width enters as an inequality constraint. 
We solve it via an exterior-point (penalty) method by gradually increasing the penalty weight~\cite{FiaccoMcCormick1990}; details appear in the Appendix.

\paragraph{Results and ablations.}
With this simple recipe, we obtain competitive quantization on ResNet-20~\cite{he2016deep} (CIFAR-10/100) and ResNet-18 (ImageNet) at W1A1, W2A2, W3A3, and near-lossless W4A4, and we also demonstrate utility for architecture exploration. 
Ablations quantify the contribution of each component.

\paragraph{Contributions.}
\begin{enumerate}
  \item A simple, effective QAT algorithm that combines LSQ, PACT, and an exterior-point constrained optimizer to realize a smooth bit-width schedule.
  \item A theoretical explanation for why this and related STE-based methods work, linking the STE assumptions with quantization noise properties (see Appendix).
  \item A theoretical explanation of distillation, KL divergence, and Jeffreys divergence in terms of noisy-channel softmax decoding (see Appendix). 
\end{enumerate}

The source code is available at GitHub.\footnote{\url{https://github.com/aifoundry-org/MHAQ/tree/GDNSQ}}

\section{GDNSQ Method}
There are multiple variants of quantization setups aimed at different goals. In this paper, our goal is to convert matrix multiplications in the inner layers to low-bit. Thus, offset quantization and first/last-layer quantization are intentionally left out of scope.

The proposed algorithm has the following stages:
\begin{enumerate}
    \item Layer replacement with quantized counterparts;
    \item PTQ to 10-bit with a min--max policy;
    \item Gradual bit-width convergence;
    \item Final LR annealing.
\end{enumerate}

\subsection{Layer replacement with quantized counterparts}
Each convolution layer with activation function \(a\),
\begin{equation}
y=a(Wx),
\end{equation}
is replaced by its quantized counterpart
\begin{equation}
y=a\big(D_w(Q_w(W))\,D_a(Q_a(x))\big),
\end{equation}
where \(Q_w,Q_a\) are the quantization functions for weights and activations, respectively (mapping reals to integers), and \(D_w,D_a\) are the corresponding dequantization functions that invert quantization up to the introduced quantization noise.

The function 
\begin{equation}
Q(x):= q(x) + r\big(q(x)\big)
\end{equation}
is the \textit{quantization function} with \textit{quantization noise}
\begin{equation}
r(x) := \lfloor x \rceil - x = \lfloor x + \tfrac12\rfloor - x,  
\end{equation}
where \(\lfloor \cdot \rceil\) denotes rounding to the nearest integer. Hence,
\begin{equation}
Q(x) =  \lfloor q(x) \rceil.
\end{equation}
Let
\begin{equation}
\bar{x}:=\operatorname{clamp}(x,l,u)=\max\,\bigl(l,\,\min(u,x)\bigr).
\end{equation}
Then the function
\begin{equation}
q(x):=D^{-1}(\bar{x}),
\end{equation}
determines the \textit{quantization level}. The function \(D^{-1}\) is the inverse of the \textit{dequantization function} 
\begin{equation}
D(x):=s x + z,
\end{equation}
where \(s\) is the quantization scale parameter
\begin{equation}
s := \frac{u-l}{2^{\omega}-1},
\end{equation}
\(\omega\) is the current quantization bit-width, and \(z\) is the offset. Thus,
\begin{equation}
D^{-1}(\bar{x}) = (\bar{x} - z)s^{-1}.
\end{equation}
Composing \(Q\) with \(D\) gives an explicit form for the reconstructed real value:
\begin{equation}
D(Q(x)) = s s^{-1} (\bar{x} - z) + s\, r\big(q(x)\big) + z = \bar{x} + s\, r\big(q(x)\big).
\end{equation}

The product \(D_w(Q_w(W))\,D_a(Q_a(x))\) is fused before inference to extract the integer matrix--vector multiplication \(Q_w(W)\,Q_a(x)\).

The function \(\bar{x}\) is differentiable almost everywhere. The remaining step in differentiating the quantizer is to define the derivatives of \(s\,r(q(x))\) under STE assumptions:
\begin{equation}
\frac{\partial (s r)}{\partial x} := 0,
\end{equation}
\begin{equation}
\frac{\partial (s r)}{\partial s} := \operatorname{Bernoulli}\bigl(\tfrac12\bigr)-\tfrac12.
\end{equation}
Only for the function \(r\) is the standard backward step overridden; all others use default autograd backward implementations. We adopt a DoReFa-style Bernoulli-noise gradient replacement \cite{zhou2016dorefa} for its computational efficiency and its alignment with SPSA \cite{spall1992spsa}, which yields unbiased gradient estimates and robustness guarantees.

The above parameterization combines LSQ and PACT, which makes the quantization bit-width \(\omega\) a differentiable parameter. As a result, \(\omega\) can be set to an arbitrary positive real value
\begin{equation}
\omega=\log_2(u-l+s) - \log_2 s=\log_2 \Bigl(\frac{u-l}{s}+1\Bigr),
\end{equation}
providing the desired smoothness in \(\omega\). Similar to the Shannon--Hartley theorem~\cite{shannon1949}, \(\omega\) can be viewed as a channel rate at quantization noise scale \(s\).

\subsection{PTQ}
To speed up convergence, QAT is applied to a pre-quantized model. We use a simple PTQ algorithm that maps values to the min--max range over the training dataset with 10-bit quantization. This yields negligible quality loss while instantly reducing the bit-width from 32 (FP32) to 10, shaving a few QAT epochs.

\subsection{Gradual bit-width convergence}
This step is performed using the standard RAdam~\cite{liu2019variance} optimizer with default parameters. RAdam is used instead of Adam to prevent instability during BatchNorm statistics accumulation. The following loss implements exterior-point constrained optimization:
\begin{equation}\label{eq:loss}
L(x,x^*)=t_q\, c_r\, P(\omega_w,\omega_a, \omega_w^*, \omega_a^*)+ t_r\, d(x,x^*),
\end{equation}
where \(t_q, t_r\) are temperatures, \(c_r\) is a running auto-scaling coefficient, \(\omega_w,\omega_a\) are the current bit-width estimate tensors for weights and activations, \(\omega_w^*, \omega_a^*\) are the target bit-widths, \(P\) is a potential function, \(d\) is a distillation distance, and \(x,x^*\) are softmax outputs of the quantized (student) and reference FP (teacher) models. For target bit-width quantization we use the temperature-growth policy
\begin{equation}\label{eq:temperature}
\bigl(t_q(n), t_r(n)\bigr) = \bigl(\lambda_n n,\, 1\bigr),
\end{equation}
where \(n\) is the current batch index in training and \(\lambda_n\) is the learning rate at batch \(n\).

For distillation, we use the Jeffreys (symmetrized Kullback--Leibler) divergence
\begin{equation}
D_{\mathrm{KL}}(P\Vert Q)
:= \sum_{x\in\mathcal{X}} P(x)\,\log\frac{P(x)}{Q(x)},
\end{equation}
\begin{equation}
d(x,y):=J(y,x)=D_{\mathrm{KL}}(x\Vert y)+D_{\mathrm{KL}}(y\Vert x).
\end{equation}

The potential function is the distance from the exterior point to the constraint surface:
\begin{equation}
P = \mathbb{E}\big(\max(0, \omega_w - \omega_w^*) + \max(0, \omega_a - \omega_a^*)\big).
\end{equation}

To simplify tuning, we implement auto-scaling of \(P\) to the running average of the distillation distance:
\begin{equation}
c_r(n) = \frac{1}{n}\sum_{i=0}^{n-1} d(x_i, x^*_i).
\end{equation}

\subsection{Final LR annealing}
To increase convergence speed, the bit-width convergence step uses a relatively high constant LR, which can leave NN parameters suboptimal by the time the target bit-width is reached. For final tuning, we perform exponential LR annealing with \(\alpha=0.9985\):
\begin{equation}
\lambda_{n+1} = \alpha \lambda_n.
\end{equation}

\section{Related works}
\textbf{Uniform STE based}: LSQ~\cite{esser2019learned} proposes a differentiable noise scale. PACT~\cite{choi2018pact} proposes  differentiable clamp bounds. PACT can be considered as an architecture change, inserting additional ReLU  units and bias to implement clamps. 
DoReFa~\cite{zhou2016dorefa} uses a random gradient in the backward pass.

\textbf{Non-uniform STE based}: nuLSQ~\cite{10.1007/978-981-96-0966-6_4} is a non-uniform variant of LSQ~\cite{esser2019learned}. LCQ~\cite{yamamoto2021learnablecompandingquantizationaccurate} is another non-uniform LSQ variant with a learnable parametric value commanding curve. IOS-POT~\cite{HE20243021}, APoT~\cite{li2019fully} are Power-of-Two non-uniform methods.

\textbf{Uniform smooth regularization}: DSQ~\cite{gong2019differentiablesoftquantizationbridging}.

\textbf{Uniform other gradient based}: RBNN~\cite{lin2020rotatedbinaryneuralnetwork}, SiMaN~\cite{Lin2023SiMaN} cosine distance-based Binary Network specific optimizations.

\textbf{Uniform search based}: SLB~\cite{yang2020searching} proposes a hybrid training/search-based algorithm for quantized weight values selection. For activation quantization, DoReFa~\cite{zhou2016dorefa} is used.

\textbf{Quantization Distillation}: The usage of the soft labels generated by the FP version of the same model instead of hard labels during QAT training is proposed by~\cite{shin2020knowledge, zhang2021learning} as QAT regularization to prevent oscillations during training. BWFR~\cite{yu2024improvingquantizationawaretraininglowprecision} in addition to distillation from the FP model, also replaces parts of the quantized model with FP counterparts, training them to work interchangeably.

\textbf{Quantization-friendly architectures}: On ImageNet, the accuracy of binarized ResNet-18 improved markedly with the introduction of Bi-Real~\cite{Liu2018BiRealNet}, which adds more skip connections within each basic block to preserve real-valued activations and increase information flow at low bit-widths. RBNN~\cite{lin2020rotatedbinaryneuralnetwork} and SiMaN~\cite{Lin2023SiMaN} explicitly consider Bi-Real–style double-skip connections and report consistent gains over the normal (vanilla) ResNet structure.

\textbf{Architecture adaptation}: NCE~\cite{park2023automaticnetworkadaptationultralow} proposes a channel splitting technique and applies its quantization bottlenecks during training.

\section{Experiments}
\subsection{Setup}
For our experiments, we decided to use the most common neural network architectures and the corresponding datasets that allow for direct comparison across papers, meanwhile maintaining a reasonable amount of computational complexity. For the models, ResNet18 \cite{he2016deep} and ResNet20 \cite{he2016deep} were chosen. And for the datasets, CIFAR-10, CIFAR-100 \cite{Krizhevsky2009LearningML} and ILSVRC2012 \cite{ILSVRC15} were chosen. The important note is that for ResNet20 there are two implementations that slightly differ from each other and allow for different results from quantization algorithms. One of them follows the original paper completely and is re-implemented by Yerlan Idelbayev \cite{Idelbayev18a} The other one, often referred to as "preactivated" ResNet, is the modified version, which achieves slightly better metrics out-of-the-box by adding a convolution shortcut with a 1x1 kernel, which is not in the original ResNet20 architecture \cite{he2016identitymappingsdeepresidual}. For the sake of fair comparison, we will be comparing them separately in order to understand the performance of the quantization method itself.

\subsection{Experiments on CIFAR-10}

In this section we provide a comparison for two slightly different models on the CIFAR-10 dataset. The proposed GDNSQ method could be used for an arbitrary amount of bits for weights and activations and can achieve a quantization configuration as low as W1A1. For the comparison, a couple of different models that achieve state-of-the-art results in a particular bit-width were chosen. 

We compare with the published results of the following methods: \textbf{PACT}~\cite{choi2018pact}, \textbf{DSQ} \cite{gong2019differentiablesoftquantizationbridging}, \textbf{SLB} ~\cite{yang2020searching}, \textbf{RBNN}~\cite{lin2020rotatedbinaryneuralnetwork}, \textbf{IOS-POT}~\cite{HE20243021}, \textbf{APoT}~\cite{DBLP:journals/corr/abs-1909-13144}.

As shown in Table~\ref{tbl:cifar10_1}, GDNSQ outperforms all counterparts in every configuration except W1A1, where it loses only to the binarization-specific \textbf{RBNN}.

\begin{table}[tbp]
\begin{tabular}{c|cc|c}
\hline
\multirow{2}{*}{\textbf{Method}} & \multicolumn{2}{c|}{\textbf{Configuration}} & \textbf{Accuracy Top-1 (\%)} \\ \cline{2-4} 
                                 & \textbf{W}   & \textbf{A}   &                              \\ \hline
FP                               & 32                 & 32                     & 92.62                        \\ \hline
DSQ                              & 1                  & 1                      & 84.10                         \\
SLB                              & 1                  & 1                      & 85.50                         \\
RBNN                             & 1                  & 1                      & \textbf{86.50}                \\ 
GDNSQ (Ours)                     & 1                  & 1                      & 85.30 $\pm$ 0.4               \\ \hline
DSQ                              & 1                  & 32                     & 90.20                         \\
SLB                              & 1                  & 32                     & 90.60                         \\
GDNSQ (Ours)                     & 1                  & 32                     & \textbf{92.01 $\pm$ 0.1}               \\ \hline
PACT                             & 2                  & 2                      & 88.9                          \\
SLB                              & 2                  & 2                      & 90.60                         \\
APoT                             & 2                  & 2                      & 91.00                         \\
GDNSQ (Ours)                     & 2                  & 2                      & \textbf{91.36 $\pm$ 0.1}                \\ \hline
PACT                             & 3                  & 3                      & 90.6                         \\
APoT                             & 3                  & 3                      & 92.20                        \\
IOS-POT                          & 3                  & 3                      & 92.24                        \\
GDNSQ (Ours)                     & 3                  & 3                      & \textbf{92.42 $\pm$ 0.04}               \\ \hline
PACT                             & 3                  & 3                      & 90.8                         \\
SLB                              & 4                  & 4                      & 91.60                         \\
APoT                             & 4                  & 4                      & 92.30                         \\
IOS-POT                          & 4                  & 4                      & 92.61                         \\
GDNSQ (Ours)                     & 4                  & 4                      & \textbf{92.64 $\pm$ 0.09}
\end{tabular}
\caption{Methods on Cifar-10 dataset using ResNet20 model architecture.}
\label{tbl:cifar10_1}
\end{table}

For the pre-activated alternative of the ResNet-20 model architecture, we used a different set of methods to compare against: \textbf{LCQ} \cite{DBLP:journals/corr/abs-2103-07156}, \textbf{nuLSQ} \cite{10.1007/978-981-96-0966-6_4}, \textbf{BWRF} \cite{yu2024improvingquantizationawaretraininglowprecision}. It is illustrated in Table \ref{tbl:cifar10_2} that the proposed method outperforms all its competitors by a margin in all the quantization configurations tested. Experiments show that GDNSQ could be effectively used to convert the FP model to a 4-bit quantized counterpart, achieving lossless quantization without any modifications to the original model architecture.

\begin{table}[ht]
\begin{tabular}{c|cc|c}
\hline
\multirow{2}{*}{\textbf{Method}} & \multicolumn{2}{c|}{\textbf{Configuration}} & \textbf{Accuracy Top-1 (\%)} \\ \cline{2-4} 
                                 & \textbf{W}           & \textbf{A}           &                              \\ \hline
FP                               & 32                   & 32                   & 93.90                         \\ \hline
GDNSQ (Ours)                     & 1                    & 1                    & \textbf{87.08 $\pm$ 0.15}    \\ \hline
BWRF                             & 2                    & 2                    & 90.98                        \\
nuLSQ                            & 2                    & 2                    & 91.30                        \\
LCQ                              & 2                    & 2                    & 91.80                         \\
GDNSQ (Ours)                     & 2                    & 2                    & \textbf{92.30 $\pm$ 0.1}               \\ \hline
nuLSQ                            & 3                    & 3                    & 92.66                        \\
BWRF                             & 3                    & 3                    & 92.76                        \\
LCQ                              & 3                    & 3                    & 92.80                         \\
GDNSQ (Ours)                     & 3                    & 3                    & \textbf{93.53 $\pm$ 0.03}               \\ \hline
BWRF                             & 4                    & 4                    & 93.13                        \\
nuLSQ (Ours)                     & 4                    & 4                    & 93.16                        \\
LCQ                              & 4                    & 4                    & 93.20                         \\
GDNSQ (Ours)                     & 4                    & 4                    & \textbf{93.90 $\pm$ 0.02}              
\end{tabular}
\caption{Methods on Cifar-10 dataset using pre-activated ResNet20 model architecture}
\label{tbl:cifar10_2}
\end{table}

\subsection{Experiments on CIFAR-100}

Train ResNet20 on CIFAR-100 dataset is a less common practice in modern papers; henceforth, it limits the amount of papers to compare to. However, evaluating identical methods on the same network, with only the dataset being altered, enables the examination of the training data's effect on the quantization process. For the CIFAR-100 dataset, the main competitor \textbf{nuLSQ} method was chosen because it is quite recent and performs well. 

Table \ref{tbl:cifar100} illustrates conducted experiments and shows that nuLSQ outperforms the proposed method at a2w2 configuration, while giving the lead in the rest of the configurations. The possible reason for that probably lies in the non-uniform approach of the method.

\begin{table}[tbp]
\begin{tabular}{l|cc|c}
\hline
\multirow{2}{*}{\textbf{Method}} & \multicolumn{2}{c|}{\textbf{Configuration}} & \textbf{Accuracy Top-1 (\%)} \\ \cline{2-4} 
                                 & \textbf{W}           & \textbf{A}           &                              \\ \hline
FP                               & 32                   & 32                   & 70.31                        \\ \hline
\multirow{3}{*}{nuLSQ}           & 2                    & 2                    & \textbf{66.02}               \\
                                 & 3                    & 3                    & 68.58                        \\
                                 & 4                    & 4                    & 69.42                        \\ \hline
\multirow{3}{*}{GDNSQ (Ours)}    & 1                    & 1                    & \textbf{57.74 $\pm$ 0.25}    \\
                                 & 2                    & 2                    & 65.98 $\pm$ 0.3              \\
                                 & 3                    & 3                    & \textbf{69.09 $\pm$ 0.2}     \\
                                 & 4                    & 4                    & \textbf{70.24 $\pm$ 0.03}    \\ \hline
\end{tabular}
\caption{nuLSQ and GDNSQ methods on Cifar-100 dataset using pre-activated ResNet20 architecture}
\label{tbl:cifar100}
\end{table}

\subsection{Experiments on ILSVRC~2012}
ResNet-18 on ImageNet-1k is a popular benchmark for evaluating quantization and binarization methods. In surveying recent SOTA results, we found that some papers report “ResNet-18” while actually using non-vanilla backbones (e.g., Bi-Real with double-skip shortcuts). For example, ReSTE briefly states in the text that it uses the Bi-Real architecture for ImageNet at W1A1 ``for fair comparison,'' and the model in the source code called ``resnet18.py'' is not the vanilla ResNet-18~\cite{wu2023estimator}. For SLB~\cite{yang2020searching}, no official code is available, so the exact backbone underlying their “ResNet-18” ImageNet results cannot be verified.

\emph{Note.} The public SLB repository configures its ImageNet experiments with a Bi-Real backbone (model name \texttt{bireal\_resnet18}) rather than the vanilla ResNet-18.\footnote{\url{https://github.com/zhaohui-yang/Binary-Neural-Networks/blob/main/SLB/README.md}}

By contrast, ReActNet clearly documents its architectural modifications relative to standard backbones~\cite{Liu2020ReActNet}. Moreover, RBNN and SiMaN analyze the impact of Bi-Real–style double-skip connections and report that such backbones outperform the vanilla ResNet in the binary setting~\cite{lin2020rotatedbinaryneuralnetwork,Lin2023SiMaN}. 

To avoid ambiguity, we restrict our main comparisons to methods with public code where the backbone is verifiably vanilla; non-compliant results (no code and/or architecture modifications) are listed in Table~\ref{tbl:imgnet-bad}.

\begin{table}[tbp]
\begin{tabular}{c|cc|c}
\hline
\multirow{2}{*}{\textbf{Method}} & \multicolumn{2}{c|}{\textbf{Configuration}} & \textbf{Accuracy Top-1 (\%)} \\ \cline{2-4} 
                                 & \textbf{W}   & \textbf{A}   &                              \\ \hline
FP                               & 32                 & 32                     & 71.93                        \\ \hline
ReSTE                            & 1                  & 1                      & 60.88                \\  
ReActNet                         & 1                  & 1                      & 65.9                        \\  
SiMaN*                           & 1                  & 1                      & 66.1                \\  
SLB                              & 1                  & 32                     & 67.10                \\ 
SLB                              & 2                  & 2                      & 66.10                \\ 
LCQ                              & 2                  & 2                      & 68.90                \\ 
LCQ                              & 4                  & 4                      & 71.50               \\ 
\end{tabular}
\caption{Performance comparison on ResNet18 with ImageNet dataset, methods without source code or using modified architecture}
\label{tbl:imgnet-bad}
\end{table}
* SiMaN accuracy of $66.1$ is specifically stated for Bi-Real architecture modification, for vanilla ResNet18 result is $60.1$ (see Table~\ref{tbl:imgnet}).

The only methods for which both reported accuracy and publicly available source code exist are \textbf{RBNN}, \textbf{SiMaN}, \textbf{DSQ}, \textbf{DoReFa}, and \textbf{PACT}. Because PACT supersedes DoReFa, we omit DoReFa from the comparison.
GDNSQ is competitive across all settings except W1A1, where its accuracy is lower. Nevertheless, it is the general-purpose method and approaches the performance of binarization-specific \textbf{RBNN} and \textbf{SiMaN}. See Table~\ref{tbl:imgnet} for the full comparison.

\begin{table}[tbp]
\begin{tabular}{c|cc|c}
\hline
\multirow{2}{*}{\textbf{Method}} & \multicolumn{2}{c|}{\textbf{Configuration}} & \textbf{Accuracy Top-1 (\%)} \\ \cline{2-4} 
                                 & \textbf{W}   & \textbf{A}   &                              \\ \hline
FP                               & 32                 & 32                     & 71.93                        \\ \hline
RBNN                             & 1                  & 1                       & 59.9                        \\ 
SiMaN                            & 1                  & 1                      & \textbf{60.1}                \\
GDNSQ (Ours)                     & 1                  & 1                      & 58.6                        \\ \hline
DSQ                              & 1                  & 32                     & 63.67                        \\
GDNSQ (Ours)                     & 1                  & 32                     & \textbf{66.28}                        \\ \hline
DSQ                              & 2                  & 2                      & 65.17                        \\
PACT                             & 2                  & 2                      & 64.4                        \\
GDNSQ (Ours)                     & 2                  & 2                      & \textbf{65.94}               \\ \hline
DSQ                              & 4                  & 4                      & \textbf{69.56}               \\ 
PACT                             & 4                  & 4                      & 69.2               \\ 
GDNSQ (Ours)                     & 4                  & 4                      & 69.50            \\ \hline
\end{tabular}
\caption{Performance comparison on ResNet18 with ImageNet dataset}
\label{tbl:imgnet}
\end{table}

\section{Ablation study}
The ablation results for most of the sections are gathered in the table~\ref{tbl:ablation} using ResNet20 CIFAR10 W1A1 quantization as a probe.

\begin{table}[tbp]
\centering
\begin{tabular}{c|c}
\hline
Method modification & \textbf{Best Top-1} \\
\hline
Baseline & 85.30 $\pm$ 0.4 \\
No PTQ & 85.14  \\
No bit-width scaling & 44.55 \\
No distillation & - \\
CE distillation loss & 83.66 \\
FP Top-1 91.67 & 84.14 \\
Batchnorm freeze & 83.38  \\
\hline
\end{tabular}
\caption{ResNet20 CIFAR10 W1A1 ablation results}
\label{tbl:ablation}
\end{table}

\subsection{Validating convergence to target bit-width}
When working with smooth regularization of quantization, the common issue is a false detection of the convergence to the target bit-width, leading to inflated metric values. We are accurately calculating the actual NN bit-width by counting unique values in weights and activations during validation, as shown in the figures~\ref{fig:activations_converge},~\ref{fig:weights_converge}. The ''estimated'' value in the loss function closely approximates the ''mean'' of the actual bit-width over the NN. The ''max'' actual bit-width can only take integer values, so it has a staircase shape and slightly lags behind the average bit-width. As we can see, when the average bit-width converges, the maximum also converges to the desired value.
\begin{figure}[htbp]
    \centering
    \includegraphics[width=0.9\linewidth]{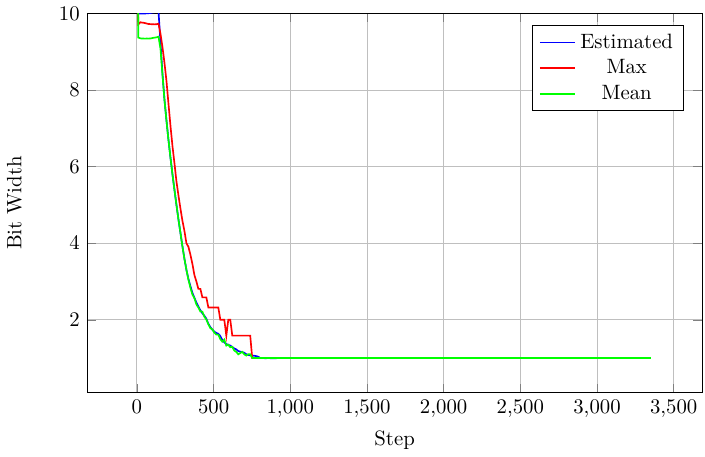}
    \caption{ResNet20 CIFAR10 W1A1 activations convergence}
    \label{fig:activations_converge}
    \Description*{}
\end{figure}

\begin{figure}[tbp]
    \centering
    \includegraphics[width=0.9\linewidth]{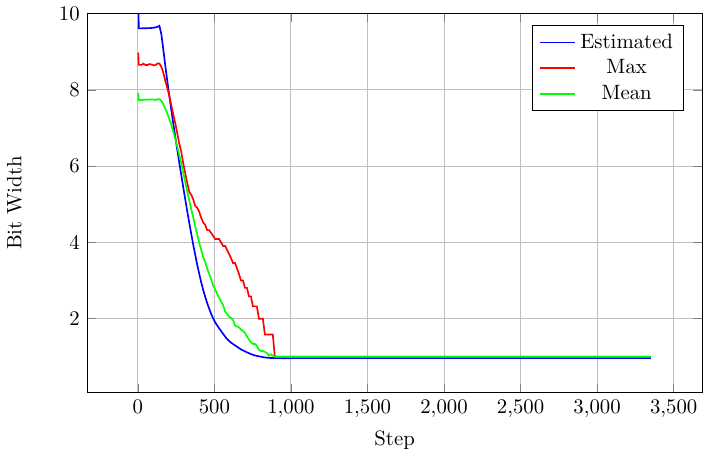}
    \caption{ResNet20 CIFAR10 W1A1 weights convergence}
    \label{fig:weights_converge}
    \Description*{}
\end{figure}

\subsection{The impact of the LR annealing}
As seen from the figure~\ref{fig:train_loss}, when the temperature grows, it causes a steady increase in the training loss due to the increased weight of potential \(P\) in~(\ref{eq:loss}). Eager converging of the bit-width leads to the overshoot in the distillation loss (figure~\ref{fig:dist_loss}). So the solution lands on the constraint surface at the suboptimal point, and the additional effort is needed to converge to the optimal point moving along the constraint. Table~\ref{tbl:lr_anneal} shows the difference between the accuracy when the bit-width target is first reached and the best value. 

\begin{figure}[tbp]
    \centering
    \includegraphics[width=0.9\linewidth]{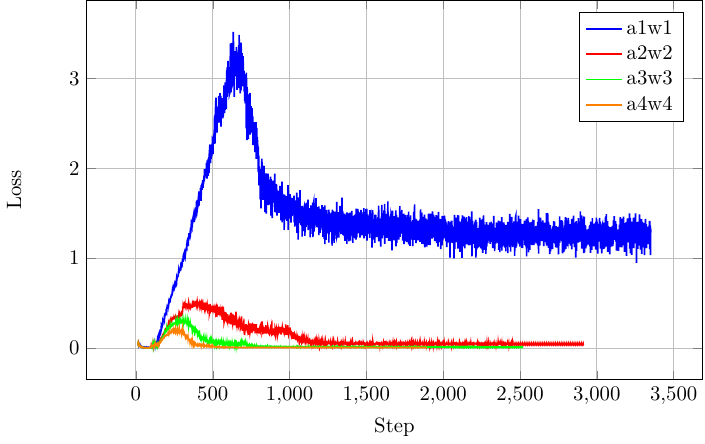}
    \caption{ResNet20 CIFAR10 training loss~(eq. \ref{eq:loss})}
    \label{fig:train_loss}
    \Description*{}
\end{figure}

\begin{figure}[tbp]
    \centering
    \includegraphics[width=0.9\linewidth]{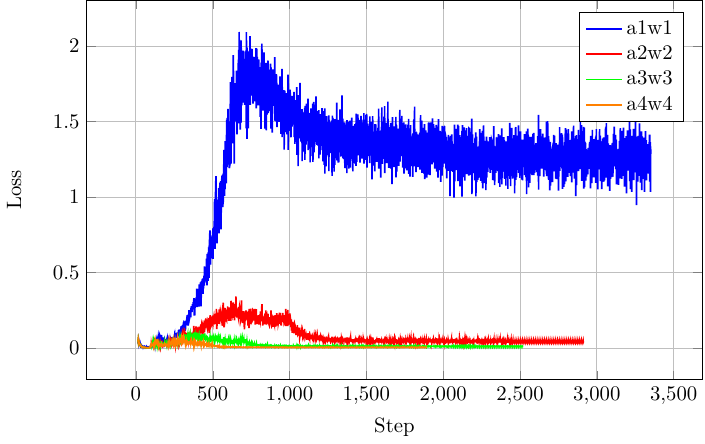}
    \caption{ResNet20 CIFAR10 distillation loss \(d\)}
    \label{fig:dist_loss}
    \Description*{}
\end{figure}

\begin{table}[tbp]
\centering
\begin{tabular}{c|c|c}
\hline
\textbf{Bit-width} & \textbf{Top-1 at bit-width} & \textbf{Best Top-1} \\
\hline
W1A1 & 68.99 & 85.16 \\
W2A2 & 85.61 & 91.24 \\
W3A3 & 89.97 & 92.52 \\
W4A4 & 91.25 & 92.53 \\
\hline
\end{tabular}
\caption{ResNet20 CIFAR10 Top-1 when the bit-with first reached vs. best}
\label{tbl:lr_anneal}
\end{table}

\subsection{The impact of PTQ}
The main reason we have included PTQ is the beautification of the figures~\ref{fig:activations_converge},~\ref{fig:weights_converge}. As seen from the table~\ref{tbl:ablation} it has minor impact on the accuracy and the convergence speed of the method.

\subsection{The impact of gradual bit-width scaling}

We disable the gradual bit-width scaling by setting the high initial value of \(t_q\) in~(\ref{eq:loss}). As a result, the model almost immediately converges to the desired bit-width and then moves along the surface of the constraint. With this modification, ResNet20 W1A1 training converges to the local optimum with very low accuracy.

\subsection{The impact of distillation}
To disable the distillation, the distillation loss is replaced by Cross-Entropy loss between predicted and ground-truth hard labels used in FP model training. Without distillation, ResNet20 training does not converge to a meaningful solution.

\subsection{The impact of distillation loss symmetry}
The choice of the symmetrical distillation loss was dictated by aesthetics and FEC decoding considerations. So we expected that asymmetrical Cross-Entropy loss typically used in distillation would provide the same results. In one experiment on W1A1 quantization, the results with asymmetrical loss are slightly worse than with symmetrical.

\subsection{The impact of FP checkpoint quality}
Due to the distillation, the quality of the quantized model is inherited from the FP checkpoint. So we checked how the checkpoint accuracy impacts the quantization results at low bit-width. The less accurate FP checkpoint produces the less accurate quantized model, though the difference is less than that between FP accuracies.

\subsection{Batchnorm freezing}

The paper~\cite{li2019fully} proposes freezing the Batchnorm running statistics for improving low-bit quantization. Many authors follow this recommendation. We have checked it with our method and see no improvement in quantized model accuracy.

\section{Architecture exploration}
As the proposed method directly varies the bit-widths, it is possible to observe the evolution of the per layer quantitation precision
during the training, as seen in the figure~\ref{fig:evolution}. This may give a hint on the difference in layers' sensitivity to quantization and bottlenecks.

\begin{figure*}[htbp]
    \centering
    \begin{tabular}{ccccc}
        \includegraphics[width=0.27\textwidth, valign=c]{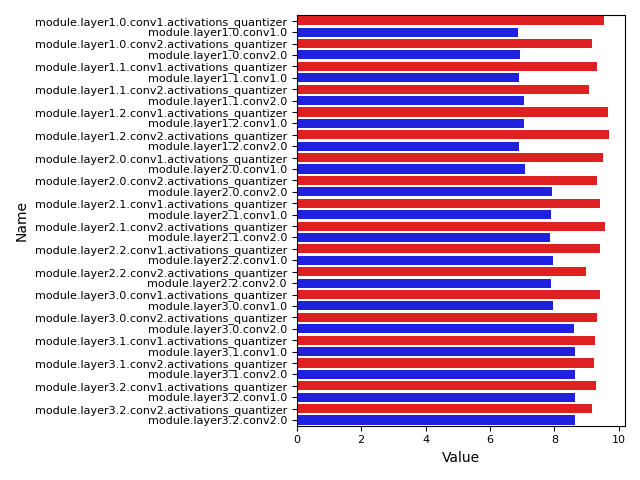}&
        {$\Rightarrow$}&
        \includegraphics[width=0.27\textwidth,valign=c]{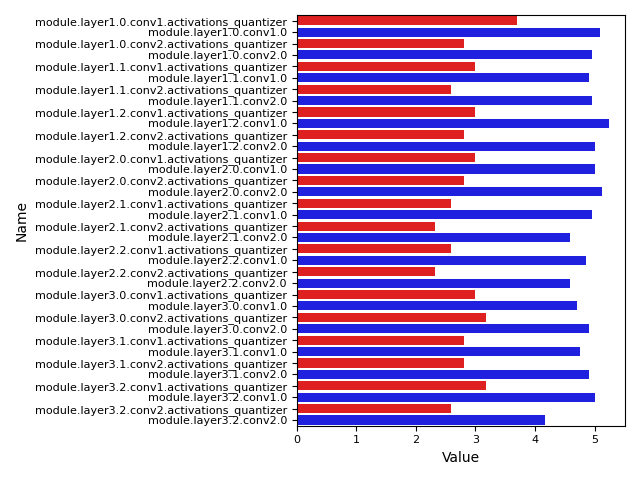}& 
        {$\Rightarrow$}&
        \includegraphics[width=0.27\textwidth,valign=c]{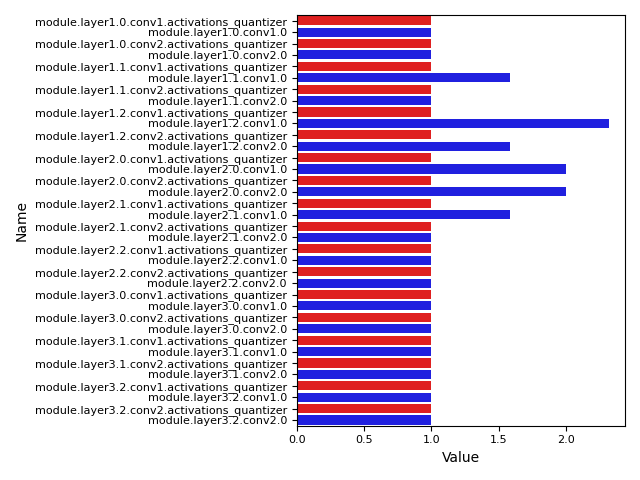}
    \end{tabular}
    \caption{Resnet20 CIFAR10 W1A1 bit-width evolution}
    \label{fig:evolution}
    \Description*{}
\end{figure*}

Another practical problem in quantization is how many bits are enough for the model to retain most of all of the FP accuracy. This problem can be solved by trying all the configurations on the grid of feasible weight/accuracy bit-width combinations, which may be expensive. The benefit of the proposed algorithm is that the loss function~(\ref{eq:loss}) is symmetrical regarding \(P\) and \(d\). So by interchanging the temperature \(t_q\) with \(t_r\) the bit-width constraint becomes an optimization target while the distillation loss becomes the constraint, 
\begin{equation}
\bigl(t_q(n), t_r(n)\bigr) = \bigl(1, \lambda_n n\bigr).
\end{equation}
So instead of quantizing to the target bit-width, the algorithm solves the problem of mixed-precision lossless quantization.

After convergence the model has the bit-widths shown on the figure~\ref{fig:bw_explore} while all the accuracy of the FP model is retained. From this figure, we can conclude that the model can be quantized to W6A7 for free except for the first quantized layer activations, which require 8 bits. This information can be used for making decisions on the model partitioning or automatic architecture adaptation similarly to NCE~\cite{park2023automaticnetworkadaptationultralow}.

\begin{figure}[tbp]
    \centering
    \includegraphics[width=0.8\linewidth]{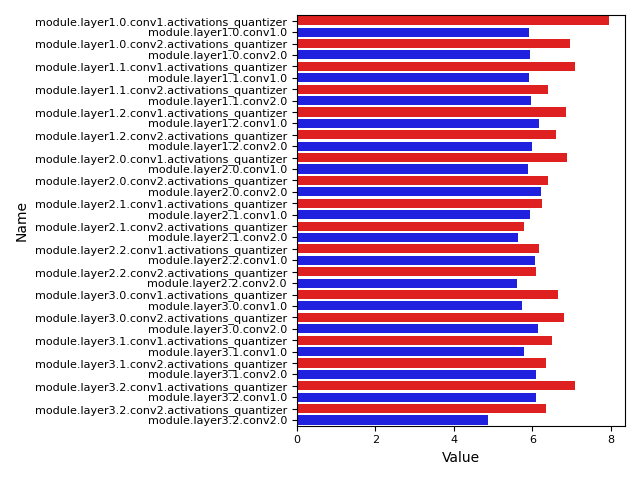}
    \caption{Resnet20 CIFAR10 lossless quantitation bit-widths}
    \label{fig:bw_explore}
    \Description*{}
\end{figure}

\section{Conclusion}
Despite being simple to understand and implement—and relying on uniform quantization—the proposed method achieves competitive results even against more complex approaches, including non-uniform, power-of-two, and search-based quantization methods. This suggests that the potential of STE-based techniques is not yet exhausted; incorporating non-uniform schemes could yield further gains. From a practical standpoint, uniform quantization aligns well with many dedicated hardware neural accelerators, especially on resource-constrained devices, broadening the method’s engineering applicability. Finally, the parallels between FEC and quantization suggest the possibility of lossless quantization relative to the FP model even at very low bit widths, analogous to lossless FEC decoding.

\bibliographystyle{acm}
\bibliography{main}

\appendix
\section{The theoretical ideas behind the algorithm}
\subsection{Calculating STE derivatives of the rounding noise}
Quantization is defined as the composition of dequantization \(D\) and quantization \(Q\).
Let
\begin{equation}
\bar{x} \;:=\; \operatorname{clamp}(x;l,u)\;=\;\max\,\bigl(l,\,\min(u,x)\bigr),
\end{equation}
\begin{equation}
D(\xi)\;:=\; s\,\xi + z,
\qquad
D^{-1}(y)\;=\; s^{-1}(y-z),
\end{equation}
\begin{equation}
q(x)\;:=\;D^{-1}(\bar{x})\;=\;s^{-1}\bar{x}-s^{-1}z,
\end{equation}
\begin{equation}
Q(x)\;:=\;q(x) + r\bigl(q(x)\bigr).
\end{equation}
The quantization noise is
\begin{equation}
r(x)\;:=\;\lfloor x \rceil - x \;=\; \lfloor x + \tfrac12\rfloor - x,
\end{equation}
where \(\lfloor \cdot \rceil\) denotes rounding to the nearest integer.
For standard integer bit-width quantization, \(q(x)\in\mathbb{Z}\) for \(x\in[l,u)\) (in particular, \(q([l,u)) \subset \mathbb{Z}\)).

By composition, we obtain an explicit form of the (de)quantized output:
\begin{equation}
D\bigl(Q(x)\bigr)
= s\bigl(q(x)+r(q(x))\bigr) + z
= \bar{x} + s\,r\bigl(q(x)\bigr).
\end{equation}
The function \(\bar{x}\) is differentiable almost everywhere (a.e.); thus, the remaining step in differentiating the quantizer is to specify the derivative of \(s\,r(q(x))\).
For a sufficiently “busy” continuous signal \(x\), the quantization error \(r(x)\) is approximately uniform on \([-\tfrac12,\tfrac12)\)~\cite{lipshitz1992quantization}, albeit generally correlated with \(x\).

By the chain rule,
\begin{equation}
\frac{\partial}{\partial x}\Bigl(s\,r\bigl(q(x)\bigr)\Bigr)
\;=\;
s\,\frac{\partial r}{\partial x}\!\bigl(q(x)\bigr)\;\frac{\partial q}{\partial x}(x).
\end{equation}
Using a centered finite-difference surrogate with a small step \(\Delta>0\),
\begin{equation}
\begin{split}
\frac{\partial r}{\partial x}(x)
\;\approx\;
\frac{1}{2\Delta}\!\Bigl(
\bigl\lfloor x+\Delta \bigr\rceil - (x+\Delta)
\;-\;
\bigl\lfloor x-\Delta \bigr\rceil + (x-\Delta)
\Bigr) \\
=
\frac{1}{2\Delta}\!\Bigl(\mathbb{\lfloor}x+\Delta\mathbb{\rceil}-\mathbb{\lfloor}x-\Delta\mathbb{\rceil}-2\Delta\Bigr).
\end{split}
\end{equation}

For integer bit-width quantization we may model \(x\) (i.e., \(q(x)\)) as uniformly distributed on an integer-length interval \([q(l),\,q(u))\).
Under mini-batch SGD we use the sample mean over the mini-batch
\begin{equation}
\mathbb{E}\,\frac{\partial r}{\partial x}(x)
\;\approx\;
\frac{1}{2\Delta}\Bigl(\mathbb{E}\,\lfloor x+\Delta \rceil
-\mathbb{E}\,\lfloor x-\Delta \rceil - 2\Delta\Bigr).
\end{equation}

\begin{lemma}
\label{lem:fd-round}
Let $l,u\in\mathbb{Z}$ with $l<u$, and let $x\sim\mathrm{Unif}[l,u)$. For any $\Delta\in(0,\tfrac12)$,
\begin{equation}
V = \mathbb{E}\,\lfloor x+\Delta\rceil-\mathbb{E}\,\lfloor x-\Delta\rceil-2\Delta = 0.
\end{equation}
\end{lemma}

\begin{proof}
Write $x=K+\xi$, where $K\in\{l,l+1,\dots,u-1\}$ and $\xi\in[0,1)$ denote the integer and the fractional parts of $x$. Since $x$ is uniform on $[l,u)$, we have
\begin{equation}
\mathbb{P}(K=k)=\frac{1}{u-l}\quad\text{for }k=l,\dots,u-1,\quad \xi\sim\mathrm{Unif}[0,1),
\end{equation}
and $K$ and $\xi$ are independent. 

Because $\Delta\in(0,\tfrac12)$ and $\xi\in[0,1)$, we have
\begin{equation}
\lfloor \xi+\Delta\rceil=\begin{cases}
0,& \xi<\tfrac12-\Delta\\
1,& \xi\ge \tfrac12-\Delta
\end{cases},
\quad
\lfloor \xi-\Delta\rceil=\begin{cases}
0,& \xi<\tfrac12+\Delta\\
1,& \xi\ge \tfrac12+\Delta
\end{cases}.
\end{equation}
Hence, by $\xi\sim\mathrm{Unif}[0,1)$,
\begin{equation}
\mathbb{E}\,\lfloor \xi+\Delta\rceil=\mathbb{P}\bigl(\xi\ge\tfrac12-\Delta\bigr)=1-\Bigl(\tfrac12-\Delta\Bigr)=\tfrac12+\Delta,
\end{equation}
\begin{equation}
\mathbb{E}\,\lfloor \xi-\Delta\rceil=\mathbb{P}\bigl(\xi\ge\tfrac12+\Delta\bigr)=1-\Bigl(\tfrac12+\Delta\Bigr)=\tfrac12-\Delta.
\end{equation}
Using $\lfloor K+\xi\pm\Delta\rceil=K+\lfloor \xi\pm\Delta\rceil$ and independence,
\begin{equation}
\mathbb{E}\,\lfloor x\pm\Delta\rceil
=\mathbb{E}\,K+\mathbb{E}\,\lfloor \xi\pm\Delta\rceil
=\frac{l+u}{2}\pm\Delta.
\end{equation}
Therefore,
\begin{equation}
V =\Bigl(\frac{l+u}{2}+\Delta\Bigr)-\Bigl(\frac{l+u}{2}-\Delta\Bigr)-2\Delta = 0.
\end{equation}
\end{proof}

By Lemma~\ref{lem:fd-round}
\begin{equation}
\mathbb{E}\,\frac{\partial r}{\partial x}(x)\ \approx\ 0.
\end{equation}
To extend this STE behavior to arbitrary \(l,u\in\mathbb{R}\) with \(l\le u\), we enforce the constraint \(\mathbb{E}\,\partial r/\partial x \approx 0\) during backpropagation.

For gradients w.r.t.\ the scale \(s\) we use the common stop-gradient assumption that \(r\) does not backpropagate through \(s\) (i.e., treat \(r\) as independent of \(s\) in the backward pass),
\begin{equation}
\frac{\partial}{\partial s}\bigl(s\,r\bigr)\;\approx\; r.
\end{equation}

\paragraph{mini-batch statistics.}
Let \(r_1,\dots,r_m\) be i.i.d.\ samples of the quantization noise with mean \(0\) and variance \(\sigma^2\).
By the Central Limit Theorem, the mini-batch mean \(\bar r_m\) satisfies
\begin{equation}
\bar r_m = \frac{1}{m}\sum_{i=1}^m r_i \approx \mathcal{N}\bigl(0,\,\sigma^2/m\bigr),
\end{equation}
so replacing the exact noise distribution by any zero-mean alternative with the same variance preserves the leading-order mini-batch behavior.

A convenient choice is a symmetric Bernoulli (Rademacher) proxy that matches the variance of the uniform rounding noise on \([-\tfrac12,\tfrac12)\), whose standard deviation is \(1/(2\sqrt{3})\).
One option is
\begin{equation}
r \;\sim\; \tfrac{1}{\sqrt{3}}\Bigl(\mathrm{Bernoulli}\bigl(\tfrac12\bigr)-\tfrac12\Bigr).
\end{equation}
In practice, one can either use the true (clipped) rounding residuals as in LSQ~\cite{esser2019learned}, or adopt such a zero-mean proxy (DoReFa~\cite{zhou2016dorefa}) with or even without variance matching.

\subsection{Relation between Symmetric KL Divergence and Hamming Distance}

Denote the Jeffreys (symmetric Kullback--Leibler) divergence between \(Q\) and \(P\) by
\begin{equation}
J(Q, P) = D_{\mathrm{KL}}(Q\|P) + D_{\mathrm{KL}}(P\|Q).
\end{equation}
Let an asymmetric binary memoryless channel (BAC) be characterized by error probabilities \( p_0 \) for the transition \( 0 \to 1 \) and \( p_1 \) for the transition \( 1 \to 0 \). For each bit, the true binary value \( b \in \{0,1\} \) specifies a “soft label” \( Q(b) \) as follows:
\begin{equation}
Q(b) =
\begin{cases}
(1-p_0,\, p_0), & \text{if } b=0,\\
(p_1,\, 1-p_1), & \text{if } b=1.
\end{cases}
\end{equation}
Let \(\tilde b\in\{0,1\}\) denote the observed (channel output) bit. Using a Jeffreys-based per-bit decision rule, define
\begin{equation}
\hat b_{j,\mathrm{J}}
:= \arg\min_{b\in\{0,1\}} J\big(Q(b),\,Q(\tilde b_j)\big),
\end{equation}
and note that the observation \(\tilde b\) induces "soft labels" \(Q(\tilde b)\).

Then, for each individual bit, the following holds:
\begin{enumerate}
    \item If no error occurred (\(\tilde b=b\)), then \( Q(\tilde b)=Q(b) \) and 
    \begin{equation}
    J(Q(b), Q(\tilde b))=0.
    \end{equation}
    \item If an error occurred (\(\tilde b\neq b\)), then
    \begin{equation}
    J(Q(b), Q(\tilde b)) = \Delta,
    \end{equation}
    where
    \begin{equation}
    \begin{split}
    \Delta = D_{\mathrm{KL}}\Bigl((1-p_0, p_0) \,\Big\|\, (p_1, 1-p_1)\Bigr) \\
          + D_{\mathrm{KL}}\Bigl((p_1, 1-p_1) \,\Big\|\, (1-p_0, p_0)\Bigr),
    \end{split}
    \end{equation}
    for the case \( b=0 \) (the case \( b=1 \) yields the same constant value).
\end{enumerate}
Thus, for a binary vector of length \( n \), the total symmetric KL divergence between the true soft labels (determined by the true bits) and the observed soft labels equals
\begin{equation}
J^{\mathrm{total}} = \Delta \cdot d_H(b,\tilde b),
\end{equation}
where \( d_H(b,\tilde b) \) is the Hamming distance between the true vector \( b \) and the observed vector \( \tilde b \), since only error positions contribute to the divergence. Hence, for an asymmetric binary memoryless channel, the symmetric KL divergence between the true soft labels and the observed soft labels is proportional to the Hamming distance between observed and true binary features.

\paragraph{Note on BSC vs.\ BAC}
If \(p_0=p_1\), the BAC reduces to a binary symmetric channel (BSC). In this case,
\begin{equation}
\begin{split}
J(Q(b), Q(\tilde b))=2D_{\mathrm{KL}}(Q(b) \| Q(\tilde b)) \\
= 2H(Q(b), Q(\tilde b))-2H(Q(b)), \quad H(Q(b)) = const,   
\end{split}
\end{equation}
where \(H(P,Q)\) denotes cross-entropy. 
Which hypothesis (BAC or BSC) is more accurate can be checked empirically by comparing which distance \(J(P,Q)\) or \(H(P,Q)\) works better.

\subsection{Quantization as the constrained optimization problem}
With the loss differentiable by bit-width, we can now formulate a smooth constrained optimization problem with inequality constraints
\begin{equation}
\begin{cases}
\bar{\Omega} = \arg\min \sum_{k=1}^{M} L(x_k,y_k^*;\Omega) \\
\omega_w(\Omega) \leq \omega_w^* \\
\omega_a(\Omega) \leq \omega_a^*, \\
\end{cases}
\end{equation}
here \(\Omega\) is a tensor of quantized NN parameters, \(M\) is a number of training samples, \(x_k\) is a training sample, \(y_k^*\) is a soft label, \(\omega_a(\Omega), \omega_w(\Omega)\) are bit-width tensors of activations and weights, \(\omega_a^*, \omega_w^*\) are target bit-widths.
\subsection{Potential method of numerically solving constrained problem as unconstrained}
As we want to gradually change the bit-width, we can either tighten the constraints over time or use the (simpler) exterior-point optimization method. To do so, we convert the constraint into a potential function that grows outside the feasible region
\begin{equation}
\bar{L}(\Omega, t)=t P(\omega_w(\Omega),\omega_a(\Omega), \omega_w^*, \omega_a^*)+ L(\Omega)    
\end{equation}
\begin{equation}
P = \mathbb{E}(\max(0, \omega_w - \omega_w^*) + \max(0, \omega_a - \omega_a^*))
\end{equation}
\begin{equation}
\bar{\Omega} = \lim_{t \rightarrow \infty} \arg\min \bar{L}(\Omega, t)
\end{equation}
This problem can be solved iteratively using SGD with a scaling parameter \(t \rightarrow \infty\) over training time.

\end{document}